\documentclass{article}
\usepackage[utf8]{inputenc}
\usepackage{fullpage}
\usepackage{amsfonts,amsmath,amstext,amssymb,amsthm,amstext,color,xcolor}
\usepackage{csquotes}
\usepackage{bbm}
\usepackage{xspace}
\usepackage[colorlinks]{hyperref}
\hypersetup{
    colorlinks = true,
    linkcolor = blue,
    anchorcolor = blue,
    citecolor = blue,
    filecolor = blue,
    urlcolor = blue
    }
\usepackage{cleveref}
\usepackage{comment}
\usepackage{parskip}
\usepackage{graphicx}
\usepackage{soul}

\theoremstyle{definition}
\newtheorem{theorem}{Theorem}[section]
\newtheorem{definition}[theorem]{Definition}
\newtheorem{remark}[theorem]{Remark}
\newtheorem{example}[theorem]{Example}

\newtheorem{proposition}[theorem]{Proposition}
\newtheorem{corollary}[theorem]{Corollary}

\newtheorem{claim}[theorem]{Claim}

\newcommand{\D}{\mathcal{D}}
\newcommand{\W}{\mathcal{W}}

\newcommand{\N}{\mathbb{N}}
\newcommand{\R}{\mathbb{R}}

\renewcommand{\i}{o}
\newcommand{\I}{I}

\newcommand{\added}[1]{#1}

\newcommand{\outpour}{efflux\xspace}
\newcommand{\efflux}{efflux\xspace}

\title{On the Learnability of Physical Concepts: \\ {\Large  Can a Neural Network Understand What's Real?} } 
\author{Alessandro Achille \and Stefano Soatto}
\date{\vspace{-0.6em}AWS AI Labs\\[2ex]%
April 13, 2022 (revised July 13, 2022)}

\begin{document}

\maketitle

\begin{abstract}

We revisit the classic signal-to-symbol barrier in light of the remarkable ability of deep neural networks to generate realistic synthetic data.  DeepFakes and spoofing highlight the feebleness of the link between physical reality and its abstract representation, whether learned by a digital computer or a biological agent. Starting from a widely applicable definition of \textit{abstract concept}, we show that standard feed-forward architectures cannot capture but trivial concepts, regardless of the number of weights and the amount of training data, despite being extremely effective classifiers. On the other hand, architectures that incorporate recursion can represent a significantly larger class of concepts, but may still be unable to learn them from a finite dataset. We qualitatively describe the class of concepts that can be ``understood'' by modern architectures trained with variants of stochastic gradient descent, using a (free energy) Lagrangian to measure information complexity. Even if a concept has been understood, however, a network has no means of communicating its understanding to an external agent, except through continuous interaction and validation. We then characterize physical objects as abstract concepts and use the previous analysis to show that physical objects can be encoded by finite architectures. However, to understand physical concepts, sensors must provide {\em persistently exciting} observations, for which the ability to control the data acquisition process is essential (active perception). The importance of control depends on the modality, benefiting visual more than acoustic or chemical perception. Finally, we conclude that binding physical entities to digital identities is possible in finite time with finite resources, therefore in principle solving the signal-to-symbol barrier problem, but awareness that the barrier has been overcome cannot be achieved in finite time by an external agent in general, thus engendering the need for continuous validation. We conduct a critical discussion of the assumptions and limitations of our analysis and indicate open avenues for future work. 
\end{abstract}

\section{Introduction}

We seek to understand the extent to which a physical entity such as an object or person can be associated to a unique identifier, represented by digits. The problem of binding physical entities with digital identities is called the ``signal-to-symbol barrier problem,'' since physical reality only manifests itself through sensory signals, whether biological or artificial. This problem has traditionally occupied philosophers and mathematicians \cite{redei2020tension}, but  has recently escaped academic circles and engulfed daily life thanks to DeepFakes, bots and spoofing. The ability of large neural networks to generate synthetic data that are indistinguishable from real has rekindled curiosity in understanding what is real, and whether it is possible to know.

To explore the binding of physical entities with digital identities, we start by defining physical entities as abstract concepts. This choice is motivated by the fact that physical entities are conceptualized by processing finite sensory observations -- even if such entities exist in the continuum. 

We then leverage well-known theoretical results on the learnability of abstract concept to argue that it is possible to bind physical entities to digital identities in finite time with finite resources. However, it may not be possible to {\em know} in finite time whether such a binding has been established. Therefore, physical identification generally requires continuous validation, and is only valid until falsified.

Furthermore, falsifiability depends on the sensory modality and inference protocol. 
For complex entities, it may not be feasible to established the binding using passively gathered data processed through convolutional neural network architectures, making these unsuitable for true physical authentication. Even interactive binding procedures may fail for sensory modalities where sources interact linearly, such as audio. 
For the sake of example, a physical entity may be an individual person, and their digital identity the union of accounts, passwords, and profiles in a closed digital system such as a computer or the Cloud. The individual is manifest to the digital system through sample observations from optical, acoustic, tactile, and other sensory modality. The binding represents the authentication process whereby access is granted to, and only to, a unique physical person.

In \Cref{sec:definitions} we describe mostly known results and, in later sections, use them to derive statements about current learning methods that include deep neural networks trained with variants of stochastic gradient descent (\Cref{sec:iterative}). Our results are summarized as follows.

\subsection{Summary of contributions}
\label{sec:contributions}

\begin{itemize}
    \item We recast classical results from \cite{gold1967} in the context of deep learning, and use them to show that feed-forward architectures commonly used today fail to capture basic concepts, such as the notions of ``even number'' (Proposition~\ref{prop:even-numbers}), or $\pi$. While the proof is straightforward using known concepts from Model Theory, it nonetheless points to such networks (including CNNs, FCNs and VAEs) 
    only being able to represent decision regions consisting of finite unions of cells contractible to a point, regardless of the dimension of the model and the volume of training data. 
    \item We also use known results to argue that architectures such as Transformers and recurrent networks can, in principle, encode any computable abstract concept (\Cref{claim:Gold}, Example~\ref{example:rnn-even}), but cannot do so with current learning methods. 
    \item We then describe a notion of \textit{awareness} of a concept, its validation and communication (\Cref{sec:awareness}). While a concept may be encoded, it may not be \textit{understood} from finite data, and even if it is understood, we may not know when that happens, or whether it has happened. In some cases, therefore, ensuring that a concept has been understood may require \textit{continuous validation}.
    \item We characterize physical objects as abstract concepts (\Cref{sec:physical}), argue they are enumerable and ``infinitesimal'' (smaller by tens of thousands of orders of magnitude) relative to the number of digital manifestations ({\em e.g.}, images, even with small resolution).
    \item As a corollary, we show that physical concepts can be understood by an active sensor (Corollary~\ref{claim:active}) and describe the conditions under which this can occur.
    Using simplified acoustic and visual data formation models, we also show that their action mechanisms differ qualitatively and, unlike visual ones, there are no benefit to acoustic interventions (\Cref{claim:acoustic}).
    \item We conclude that it is possible to bind physical entities to digital identities in finite time, thus overcoming the signal-to-symbol barrier (\Cref{sec:signal-symbol}), but in general one cannot forgo continuous validation. 
\end{itemize}
The implication for physical authentication is that a passive agent cannot establish a binding between a physical entity and their digital identity. Therefore, the authenticating agent will be forced into a cat-and-mouse game by continuously increasing the size of the dataset to incorporate out-of-distribution samples, leaving the agent always one step behind the attacker. An active agent, on the other hand, can in theory terminate the cat-and-mouse game by binding a physical identity to a digital signature in finite time. However, an external user cannot know when such a binding has occurred and will therefore need continuous reassurance from the authenticating agent that the binding is secure. This is still a cat-and-mouse game, but one where the authenticating agent is always ahead of the attacker.

Our analysis has several limitations, which we detail in \Cref{sec:discussion} and in an expanded discussion in the appendix.

\subsection{Related work}

Our manuscript touches upon a broad literature too vast to review in a single paper, so we limit our discussion to specific references we draw methods from, and to work that specifically and explicitly attempts to formalize the binding of digital and physical entities. 
An important precursor of our work is Gold's \cite{gold1967}, who studied learnability of a language, which is a case of abstract object; \cite{gold1967} shows that the class of context-sensitive languages is learnable from observations, but not regular languages. An overview of the literature on the signal-to-symbol barrier problem up to 1995 is given in \cite{bajcsy1995signal}. The ability of a neural network architecture to encode a concept, and the ability of a training scheme to learn it, are related to the classical concepts of expressiveness and learnability, on which there is a long history \cite{cybenko1989} and dedicated conferences. Our discussion of active and interactive inference relates to causal interventions \cite{pearl2009causality}, although we do not develop that relation here.

More recently, \cite{perez2021attention} have shown that Transformers with discrete weights are Turing complete, a fact that we draw on in our analysis; \cite{katharopoulos2020transformers} also shows the equivalence of Transformers and recurrent neural networks, a fact that we use to extend our results across these two classes of architectures. Both classes of architectures can express the language of first-order logic, which cannot directly capture some basic problems ubiquitous in perception such as latent-variable hard subset selection ({\em e.g.}, robust linear regression with a preponderance of high-cost outliers \cite{fischler1981random}).  It has been argued \cite{pantsar2021descriptive} that  human cognition can capture second-order logic, despite it modeling intractable computational problems,  and that bounded resources in the brain should not be equated with restriction to P-class complexity problems. We do not touch upon the vast literature in biological perception, cognitive neuroscience, and philosophy, which are well beyond our scope here, but discuss open problems throughout the manuscript, in \Cref{sec:discussion}, and in the appendix.

\section{Preliminaries}
\label{sec:definitions}

This section is based, for the most part, on established work of others and could be skipped by those familiar with the topic of learnability of abstract concepts. However, we do cast known results into a language that is suitable to address physical concepts in the rest of the manuscript using tools from Deep Learning.

\subsection{Abstract Concepts and their \efflux}

An abstract concept, or more simply a \textit{concept}, is an element $\omega \in \Omega$ of a \textit{concept class} $\Omega$. 
Each concept manifests\footnote{We restrict our attention to \textit{ primary}, or \textit{manifest concepts}, and exclude from consideration derived or deduced concepts that do not have any manifestation.} itself through a collection\footnote{We allow the possibility of repeated observations, indicated by $(\dots)$ instead of $\{\dots\}$.} 
of observations $(\i_1(\omega), \i_2(\omega), \ldots)$ where each observation $\i_t(\omega)$ belongs to a set  $I$ (codomain).
In particular, each observation can be a datum $\i_t(\omega) = x_t(\omega) \in I$, or a datum and a label $\i_t(\omega) = (x_t(\omega), y_t(\omega)) \in I \times Y$. Without loss of generality, we restrict the labels to be binary $y_t(\omega) \in Y = \{0, 1\}$. In the following, we consider the case where labels are given (supervised) unless specified otherwise.
We call the set of all possible data $I(\omega) \subset I$, through which the concept is manifest, its \textit{\efflux}.  Note that this is not just a finite collection of observations, $\D = \{\i_t(\omega)\}_{t=1}^N$, which is instead called a \textit{dataset}. The \efflux is an infinite set  containing all possible datasets that {\em could} emanate from the concept given infinite time and resources. The difference between \efflux and datasets is central to our discussion.

Observations can be produced by a sequential mechanism, in which case  $t$ denotes a discrete time step. When the data $\i_t(\omega, u)$ depends on an auxiliary variable $u$ that can be influenced or controlled, we call the data generation mechanism \textit{active}. 
If the mechanism will, given infinite time, produce an \efflux, we call it \textit{complete}, and otherwise \textit{partial}. If the mechanism only produces observations with positive labels $y_t(\omega) = 1$, regardless of whether it outputs the label itself, we call it \textit{one-sided}.

If two concepts $\omega$ and $\omega'$ have identical \outpour $I(\omega) = I(\omega')$ they are called \textit{indistinguishable}. We say that a concept $\omega$ is \textit{identifiable} in $\Omega$ if it is distinguishable from all others in $\Omega$,
that is, if for all $\omega' \in \Omega \setminus \{\omega\}$ we have $I(\omega) \neq I(\omega')$. 
Concepts that yield an \efflux that is a finite union of topological cells are called \textit{finite}\footnote{Any such concept yields an \efflux that is homeomorphic to a finite simplicial complex, with a definable homeomorphism \cite{dries_1998}.}  or (topologically) \textit{trivial} in the sense that they correspond to finite unions of decision regions in the \efflux. %
We should note that even relatively simple concepts such as ``even/odd number'' are non-trivial (Proposition~\ref{prop:even-numbers}).

Examples of concepts include laws of physics abstracted from observations, semantic labels attached to images by human annotators, the identity of an individual manifest in multiple sensory measurements, or a function implicitly defined by the solution of an optimization problem.  We are particularly interested in concepts that arise from the physical world. These comprise the ``internal representation'' of the environment (or scene)  in which we are immersed, which we can only experience through finite and noisy sensory measurements. We call these {\em learned concepts} since they are inferred from data. While abstract concepts are disembodied, physical concepts subtend action in the physical world, which responds with a \textit{reality check}. Whatever the ``true world'' may be, it is the one that responds to actions taken by an embodied agents, regardless of whether it is biological and artificial. If such actions are driven by an abstract concept, the world works to close the perception-abstraction-action loop. 

\subsection{Encoding of a concept} 

An identifiable concept can be represented by its \outpour, which is infinite in general. This representation is not suitable for learned concepts, since a learner necessarily has finite resources. Instead, we use an encoding of the characteristic function of the \outpour to represent a concept. This assumes that the \outpour $I(\omega)$ is computable\footnote{This assumption has consequences beyond excluding pathological cases, as we will discuss in Example~\ref{claim:not-computable}.}  for all $\omega \in \Omega$. 

\begin{definition}[Encoding of an abstract concept]
\label{def:encoding}
We call \textit{discriminative encoding}  of a concept, or \textit{discriminator}, a program $p$ (in the sense of Turing) that terminates on all inputs $x \in I$ and such that $p(x) = 1$ if and only if $x \in I(\omega)$ and 0 otherwise. We call \textit{generative encoding} of a concept, or \textit{generator}, a program $p$ that eventually lists all elements $x$ in the \outpour  $I(\omega)$. We say that a discriminative encoding $p$ is \textit{compatible} with the supervised observations $\i^t(\omega) = (\i_1(\omega), \ldots, \i_t(\omega))$ if, for all $\tau < t$,  we have $p(x_{\tau}) = y_{\tau}$.
\end{definition}

Note that a discriminative encoding enables testing future data, akin to what \cite{gold1967} calls a \textit{tester}, and corresponds to a recursive \efflux. A generative encoding corresponds to a recursively enumerable \efflux, and can represent strictly more concepts than a discriminator, as we show next. On the other hand, generators are not suitable to being used as testers, even if they can represent more concepts. We call the set of encodings, whether generative or discriminative, $\W$. 

\begin{remark}[Generators are more powerful than discriminators] A discriminative encoding can be converted into a generative one but not vice-versa. Let $p: \I \to \{0, 1\}$ be a discriminator and let $h: \N \to I$ be an enumeration of $I$. Construct a generator $g: \N \to I$ as follows: Let $n_1 \in \N$ be the first index in the enumeration $h$ such that $p(h(n_1)) = 1$, $n_2$ the second and so on. Then $g = h(n_s)$ is a generator. To see that the inclusion is strict, consider a program implemented by a Turing machine, and the abstract concept $\omega = \{\texttt{programs that terminate on an empty tape}\}$. The generator runs the first $t$ machines for $t$ steps, increases $t$ until it finds one that terminates, and outputs it if not done so before. This will eventually output all the Turing machines that terminate. This concept does not have a discriminator, as deciding if a Turing machine belongs to $\omega$ is equivalent to the halting problem. 
\end{remark}

\subsection{Learning and understanding a concept} 
\label{sec:understanding}

Concepts can be represented by \textit{digital entities}, or symbols.\footnote{For instance the value of the weights of a neural network stored in a computer. A set of \textit{symbols} shared among multiple agents is called a \textit{dictionary}, whose elements can be combined in the process of reasoning or \textit{deduction}. The process of communicating a concept, that is to put different entities into correspondence, is called an \textit{explanation}. Since communication and reasoning are beyond our scope, we use the terms concept, digital entity and symbol interchangeably and restrict our attention to \textit{primary concepts} that can be derived from data, rather than combined deductively.}
We call \textit{abstraction} the mapping of a finite set of data to a concept, and  \textit{induction} the  mapping of the abstracted concept to an \efflux.\footnote{Abstraction is also called conceptualization or symbolization. It is a process of creation that involves no surprise. Understanding, on the other hand, is a process of discovery that leads to the signal-to-symbol barrier being overcome. Induction is equivalent to generalization so, by definition, understanding implies infinite generalization.} The process of associating a symbol to a specific datum is called \textit{grounding} and the process of associating a specific datum to a symbol is  called \textit{binding}.\footnote{A finite set of data can be associated to multiple concepts, as many as can be inferred from shared latent attributes of the data. Conversely, however, \textit{a label attached to a finite set of data is not sufficient to define a concept.} The label tautologically defines a relation among the given data that share it, but it may not represent a concept due to the presence of \textit{nuisance variability}, leading to  \textit{overfitting}. For example, a finite collection of images of different physical scenes that a human annotator labeled as ``kitchen'' could also represent ``indoors'' or ``white stuff'' or ``furniture'' or ``images.'' From the dataset alone, it is not possible to extend the concept ``kitchen,'' which exists in the brain of the annotator, to the \outpour of all unseen images of kitchens, unless the dataset spans all possible factors of variation shared by all images in the class ``kitchen,'' but not ``furniture'' or ``white stuff.'' To do so, the dataset would have to include images of kitchens indoor and outdoor, viewed under different lighting, vantage point, and made in different shapes with different materials. The symbol ``kitchen'' then would be the discrete invariant that subtends the infinite \outpour, and the finite dataset would be equivalent to the \outpour, in a sense that we will make precise later.}
The process of inferring a symbol from a signal inductively is called ``understanding.'' We use the term understanding instead of ``learning'' for the inductive process, since capturing the concept, even from finite data, affords infinite generalization to the entire \efflux, whereas learning generally refers to inferring a discriminant from a finite dataset, with varying degrees of generalization.

\begin{definition}[Learner]
A \textit{learner} $F: O^* \to \W$ is a function that maps a dataset of observations \textit{up to index} $t$, $\i^t(\omega) = (\i_1(\omega), \i_2(\omega), \ldots, \i_t(\omega))$ onto an encoding $p_t = F(\i^t(\omega))$ 
of the concept $\omega \in \Omega$. 
\end{definition}
Under the assumption of identifiability, an abstract concept $\omega$ can be uniquely determined from \textit{infinitely many} observations $(\i_t(\omega))_{t=1}^\infty$.
However, we are interested in the case where the concept $\omega$ can be identified using only a dataset of \textit{finitely many observations} $\i^t(\omega)$. This notion is formalized by the following definition. 
\begin{definition}[Understanding]
\label{def:understanding}
Let $F$ be a program and $x_1, x_2, \ldots$ a sequence of data. We write $F(x^t) \to z$ if there is a $T\in \N$ such that for all $t \geq T$ we have $F(x^t) = z$. We say that a learner $F$ \textit{understands} the concept $\omega$ if $F(\i^t(\omega)) \to p_\omega$ where $p_\omega$ is an encoding of the concept $\omega$.
\end{definition}
A learner that has understood a concept has, in finitely many steps, captured the mechanism underlying the infinite \efflux. For example, the concept of $\pi$ can be captured in finitely many observations, where each observation could amount to observing a finite set of the initial digits, even though its \efflux (the set of all finite initial strings of digits) is infinite. On the other hand, there is no concept underlying a truly random sequence, and therefore there is nothing to be understood by observing it. 
We sometimes call the \efflux of a concept $\i(\omega)$ \textit{signal}, and the encoding of the learned concept $w$ \textit{symbol}. Then, understanding entails encoding potentially infinite signals with finitely encoded symbols. Signals can also be measurements of physical entities, and symbols their corresponding digital identities, in which case understanding means binding a physical entity to a digital one and addresses the signal-to-symbol barrier problem formally. 

We now ask whether such binding is even possible,  depending on the class of concepts $\Omega$ and the class of learners $\mathcal{F}$; that is, whether  there exists a learner  $F \in \mathcal{F}$ that can understand every concept $\omega \in \Omega$. 
To that end, we restrict\footnote{This restriction appears innocuous, since a finite computer can only store a finite class of concepts that can be enumerated. However, in \Cref{sec:physical} we will  discuss physical concepts that, at first sight, appears to evade enumerability.} ourselves to classes of concepts such that the concepts inside the class are \textit{enumerable}\footnote{That is, there is a program that terminates for all inputs and given $n$ outputs the $n$-th concept in the class} and let $p_n$ for $n \in \N$ be an enumeration of the possible encodings of the concepts in $\Omega$. In particular, we consider the class of concepts whose efflux is computable by a primitive recursive function \cite{brainerd1974theory}. While this is not the complete set of all computable concepts (which is not enumerable), it is large enough to contain most cases of practical interest.

\begin{definition}[Guessing by enumeration: Gold Learner \cite{gold1967}] 
\label{def:enumeration}
Given an enumeration of concepts in a class, there is a canonical way to construct a learner $G: \I \to \Omega$. Let $\i^t(\omega)$ be the observations up to now; let  $n_t \in \N$ be the first index such that the concept $p_{n_t}$ is compatible with the observations. Then, let $G(\i^t(\omega)) := p_{n_t}$.
\end{definition}
Gold refers to the author of this learner \cite{gold1967} as well as to its unattainability.

\begin{proposition}[Gold, 1967 \cite{gold1967}]
\label{claim:Gold}
Assume that all concepts in $\Omega$ can be encoded,  that the encodings can be enumerated and that the concepts are identifiable given the \efflux. Then, a guessing-by-enumeration learning rule $G$ (Gold Learner) can understand all concepts in $\Omega$. Moreover, there is no guessing rule uniformly faster than $G$.
\end{proposition}

\begin{proof}[Sketch of proof.] 
Suppose that there is a concept $\omega$ that a learner $G'$ guesses faster than $G$. That is, there is a $\tau$ such that $G'(\i^\tau(\omega)) = \omega$ but $G(\i^\tau(\omega)) = \omega' \neq \omega$. But, by construction of $G$, this means that $\omega'$ is also compatible with all information seen until time $\tau$, that is, $\i^\tau(\omega') = \i^\tau(\omega)$. Now, suppose we are trying to learn $\omega'$. By what we just said, the initial sequence of information  $\i^\tau(\omega')$ will be the same, so $G$ will now guess the right concept $\omega'$ at time $\tau$, while $G'$ will guess the wrong concept $\omega$. That is, $G$ is faster than $G'$ on $\omega'$ and in particular $G'$ is not uniformly faster than $G$.
\end{proof}

Given the above, we will use  ``identification-by-enumeration'' (Gold Learner) $G$ as a paragon that is utterly useless in practice, since it requires unbounded and rapidly growing space and time resources. To move towards relevance in practice, we will now focus on whether a concept class can be \textit{effectively} learned; \textit{i.e.}, whether there is any program at all that is able to learn the concept. Before doing so, in the next example we illustrate the existence of concepts that are not even computable, despite being definable.

\begin{example}[Incomputable concepts]
\label{claim:not-computable}
\footnote{\url{https://math.stackexchange.com/questions/1266587/example-of-uncomputable-but-definable-number}}
Let $\phi_n$ enumerate the sentences in the language of arithmetic. Now consider the real number $\alpha$ whose $n$-th digit in the decimal expansion is 1 if and only if the sentence is true in the language of arithmetic, ${\mathbb N} \models \phi_n$, and 0 otherwise. The number $\alpha$, or equivalently the function $f(n) = \alpha_n$ where $\alpha_n$, is a well-defined concept in the Zermelo–Fraenkel theory of sets and can be expressed in a finite way (literally, the definition above). However, it is not computable as doing so would require solving the halting-problem, since $\phi \equiv \texttt{The $i$-th Turing machine terminates}$ is a valid proposition in the language of arithmetic. A computer can store the definition of $\alpha$, use it to prove new statements (reasoning), but cannot compute its digits and cannot learn the concept from examples. Whether it is possible for a computer to learn the definition from the examples is an interesting and open model theoretic question. It is, however, not germane to our discussion as we restrict our attention to primary concepts, learned from data, whereas $\alpha$ is a symbol defined in terms of other symbols and rules consistent with an artificial set of axioms (the Zermelo-Fraenkel theory). However, we should note that we cannot exclude the possibility that   $\alpha$ may have a physical embodiment: One could discover a particle that changes spin every second and the $n$-th spin is decided based on $\alpha_n$. But this would violate the Church-Turing thesis. If we accept the Church-Turing thesis, then all physical concepts discussed in \Cref{sec:physical} must be computable, even if not learnable or understandable.
\end{example}

\subsection{Interactive learner} 

In \Cref{sec:definitions} we introduced the notion of \textit{active} observation mechanism. In \Cref{sec:physical} we will describe active sensors are measurement devices, actuators and computational procedures  that can, at least in part, control the data acquisition process. 
A learner that has access to, and can control, an active sensor is called an \textit{active learner.} When the active learner is engaged in an iterative procedure that can, in principle, generate infinite data, we call it \textit{interactive}. 

An interactive learner can control the observations $\i_t(\omega,u)$ by issuing a control action, a challenge, or an intervention $u$ over the environment. Without loss of generality we can consider all learners to be active, and organized in increasing level of control authority, from none (passive learner) to total (complete learner) \efflux control.\footnote{%
In the case of visual measurements, the hierarchy may start with a passive camera, then include control of the electronics (to mitigate uncertainty due to sensor saturation), then of the optics (to mitigate optical aberrations), then of the mechanics (to control viewpoint and mitigate limits to visibility), then of the photometry (to mitigate uncertainty in the illuminant). Each has limits, including time as often measurement devices must integrate readings over a finite temporal window that has lower bounds imposed by the underlying physics.}

From the perspective of understanding a concept, what matters is whether the the \efflux is complete, regardless of whether the learner is active of passive. However, an active learner can monitor whether the observations are \textit{persistently exciting}\footnote{A persistently exciting input is one that excites all the modes of the system.}, which a passive learner cannot do. Instead, a passive learner can only wait and hope that, eventually, the observations will cover the \efflux \added{which may require a prohibitively long time. An example of efficient interactive learning is the 20-question game, where a surprisingly large number of concepts can be identified efficiently by binary partition \cite{geman1996active,jedynak2013game}. A learner that passively listens to the answer of all questions in the English language in lexicographic order will eventually obtain the same information as the active learner, but in a much longer time. }

In the next section we discuss an iterative (local) learner, distinct from an interactive learner described here.  
The key difference is the ability to influence the data formation process. This difference will become important when deciding whether a physical concept can be understood, which we will address in \Cref{sec:signal-symbol}, and thus defer further discussion about active learners until then.

\subsection{Local Learning and Differentiable Programming}
\label{sec:iterative}

Section~\ref{sec:definitions} formalized learnability using a Gold (enumerative) learner (Definition~\ref{def:enumeration}) that is unrealistic under any practical circumstance. 
Instead, current methods encode programs using a set of parameters (weights) defined in the continuum and learned by optimizing a  differentiable loss function, all eventually quantized. Specifically, let $f_w: I \to \R^{|Y|}$ be a family of computable functions parametrized by a set of weights $w \in \W \subset \mathbb{R}^D$, with $f_w(x)$ differentiable with respect to $w$ almost everywhere on $\mathbb{R}^D$. Given a finite approximation $\hat{w}$ of the weights $w$ (\textit{e.g.}, as floating point numbers), we can consider the computer program 
\[
p_{\hat{w}}(x) = \arg\max_{c} f_{\hat{w}}(x)_c
\]
where the subscript $f_c$ denotes the $c$-th component of the vector $f$.
Given a dataset of observations $\D = \{\i_t(\omega)\}_{t=1}^N$ of a concept $\omega$, the average loss or risk is $L_\D(w) = \frac{1}{N}\sum_{t=1}^N \ell(f_w(x_t), y_t)$  where $\ell(z, y)$ is a per-sample loss, \textit{e.g.}, the (empirical) cross-entropy loss where
\begin{equation}
\label{eq:cross-entropy}
    \ell(z,y) = \sum_{c \in Y} \mathbbm{1}[y=c] \log \big(\operatorname{softmax} (z)_c \big)%
\end{equation}
where the $\operatorname{softmax}$ of a vector $z$ is the vector with components $\operatorname{softmax}(z)_i = e^{z_i}/(\sum_{j} e^{z_j})$.
The function $f_w(x)$ is called a \textit{discriminant}. 

Learning by optimizing the loss $L_\D(w)$ through a gradient descent scheme is called \textit{Differentiable Programming}, or \textit{Deep Learning} if the function $f_w$ is implemented by a deep neural network.

\begin{remark}[The minimizer of the loss $L_\D$ is not the optimal discriminant]
\label{rem:overfitting}
Even though the softmax maps to a normalized discriminant vector that can be interpreted as a probability, and even if $(x_t, y_t)$ are samples from a joint distribution $p(x_t,y_t)$, the optimal (Bayesian) discriminant, which is the posterior probability $p(y|x_t)$ is not a minimizer of the loss $L_\D$ with $\ell$ in Eq.~\eqref{eq:cross-entropy}. Instead, for unregularized problems, the minimizer is a degenerate (\textit{overfitting}) distribution that places all the mass around the samples in $\D$.
\end{remark}

The two key questions, then, are what classes of concepts can be encoded (\Cref{sec:representable}), and what concepts can be understood (\Cref{sec:dp-undestanding}), by Differentiable Programming.  We tackle these questions in the next sections, after we introduce the class of functions currently in use to implement the discriminant $f_w$.

\subsubsection{Feed-forward and recursive Neural Networks }

Neural networks are a parametric class of functions $f_w:X\rightarrow \R^{|Y|}$  from some input (data) $x\in X$ to a score vector $f_w(x) \in \R^{|Y|}$. The parameters $w$ are called \textit{weights}. The output is obtained by composing intermediate functions (\textit{layers}), $f_{w_l}$, whose outputs $x_{l+1} = f_{w_l}(x_l) = f_{w_l} \circ f_{w_{l-1}}(x_{l-1}) \ldots f_{w_0}(x)$ are called \textit{activations}. The first layer input is the datum $x_0 = x$ and the last layer activation is the discriminant vector, $x_{L} = y$. Each layer is typically an affine transformation of its input, where the coefficients of the affine map $w_l = (A_l, b_l)$ are the weights of that layer, followed by a component-wise \textit{non-linearity} $\sigma$, $f_{w_l}(x_l) = \sigma(A_l x_l + b_l)$. Such non-linearity $\sigma:\R \to \R$ can be a rectified linear unit (ReLU) $\sigma(x) = \max(x,0)$, a sigmoidal function \cite{cybenko1989}, hyperbolic tangent, or other exponential -- typically followed by normalization ($\operatorname{softmax}$). Modern neural networks are \textit{deep}, meaning that they consists of multiple layers $L \gg 1$, and are typically \textit{overparametrized}, in the sense that $|w| \gg |\D|$,  the dataset from which the weights are inferred as part of the training procedure. Deep neural networks (DNNs) that impose no constraints on the weights are called fully connected (FC). Those that impose a Toeplitz structure are called convolutional neural networks (CNNs) \cite{sermanet2012convolutional}. 

While modern deep networks may implement skip connections that bypass some layers, they do not typically implement recursions, or feedback: The output of each layer is input to the following layers, and never fed back to earlier layers in the chain. We call such networks \textit{feed-forward}. Networks that explicitly incorporate feedback or other forms of recursions are called recurrent neural networks (RNNs)  \cite{sutskever2011generating}.

A popular class of functions that is not explicitly recurrent but still incorporates recursive mechanisms are  Transformers \cite{vaswani2017attention}, which process a sequence of input \textit{tokens} using an non-local attention mechanism \cite{buades2005non}, as opposed to the local processing of convolutional networks. While transformers do not explicitly incorporate recursion, they may be used in a recursive fashion by repeatedly running the transformer using its output tokens as input to the next step in order to generate a sequence of arbitrary length.

\subsubsection{Concepts representable by Differentiable Programming}
\label{sec:representable}

What concepts are representable by an element of a family $f_w(x)$ of differentiable functions (architecture) depends on the family itself.

\textbf{Turing-complete architectures.} A number of architectures have been explicitly designed to be Turing-complete, so they can encode concepts in  \Cref{def:encoding} \cite{brainerd1974theory}. It has also been shown that both Recurrent Neural Networks and Transformers can be used to simulate a Turing machine \cite{perez2021attention}. However, these architectures often expect a sequence of discrete symbols as input (akin to the tape of a Turing machine) rather than a signal sampled from the continuum for which there is no natural discretization, such as an image (\Cref{sec:visual}) or sound (\Cref{sec:acoustic}). Rather than the input being encoded as a sequence of bits (pixels), Neural Turing Machines expect  symbols (tokens) as input. Treating pixels as symbols yields input sequences many million-long, even for small images (say $256\times 256$), and fails to capture the natural statistics and structure of images.  
While theoretically interesting, Turing-complete architectures are too generic, so we focus on specific aspects of current architectures that process continuous signals.

\textbf{Recursion.} A key characteristic of an architecture is whether it incorporates an unbounded recursion mechanisms. Next, we show that feed-forward architectures, while remarkable at data association and finite pattern recognition \cite{sermanet2012convolutional}, cannot perform abstraction, in the sense that they can only represent trivial concepts as defined in \Cref{sec:definitions}. For example, they cannot natively represent the seemingly obvious concept of ``even'' or ``odd'' number beyond a finite set.\footnote{One could trivially design a custom architecture for the task, for instance one that looks only at the last digit regardess of the length of the sequence. However, this has to be manually engineered and cannot be learned by a generic architecture with current learning methods.}

\begin{proposition}[No feed-forward architecture can encode the concept of ``even number'']
\label{prop:even-numbers}
Let $f_w(x)$ be any architecture within a differentiable program, using linear operations, ReLU non-linearities, and any function that can be defined using the exponential (\textit{e.g.}, softmax, ReLU, tanh, sigmoid, \ldots) but without recursions. Then, $f_w(x)$ cannot encode the concept of ``even number.'' That is, there is no $w$ such that $f_w(2n) \geq 0.5$ and $f_w(2n + 1) < 0.5$ for all $n \in \N$. In particular, a feed-forward network cannot partition arbitrarily large numbers into either even or odd.
\end{proposition}
\begin{proof}
Suppose that such an $f_w(x)$ exists. Following the hypotheses, the function $f_w(x)$ is definable in $M = (\R,~+,~\cdot,~\leq,~\exp)$. By \cite{dries_1998}, the theory of $M$ is order-minimal (o-minimal) and in particular the only sets definable in one dimension are \textit{finite} unions of intervals. However, the set $A = \{x : f_w(x) \geq 0.5\}$ is definable and, by hypothesis, it cannot be written as a finite union of intervals since that would imply that there is a $C$ such that either all $x > C$ are in $A$ or are not in $A$, contrary to the hypothesis.
\end{proof}
By contrast, an architecture that incorporates recursion (or that can tokenize the input as a string of variable length and ignore all but the last digit) can easily partition arbitrary numbers into even or odd, as we show in the next example.
\begin{example}[An RNN that encodes even numbers]
\label{example:rnn-even}
Let the state $s_t = (n_t, a_t, e_t)$ at time $t$ represent the current counter $n_t$, a parity flag $a_t$, and an end of sequence tag $e_t$. Define the recursive linear update
\[
s_{t+1} \gets (n_t-1, 1-a_t, 1-\operatorname{sign}(n_t)), 
\]
where the $\operatorname{sign}$ function can be approximated using a sigmoidal non-linearity. Initialize the  state with $s_0 = (n, 0, 0)$ where $n$ is the number whose parity we want to determine. Then, the recurrent network terminates after $n+1$ steps ($e_t$ becomes 1) and the value of $a_t$ at the last step encodes the parity of the number $n$.
\end{example}

Using the same reasoning as \Cref{prop:even-numbers} we can prove the following characterization of all concepts that can be learned by a feed-forward network.
\begin{proposition}[Concepts encodable by a feed-forward network]
\label{prop:encodable}
Let $f_w: \R^n \to \R$ be a feed-forward network satisfying the same condition as \Cref{prop:even-numbers}. Let $A = \{x : f_w(x) < 0.5\}$ be the set of examples classified as positive. Then, $A$ is a \textit{finite} union of topologically trivial cells (contractible to a point).
\end{proposition}
Conversely, \Cref{claim:Gold} can be used to show that a transformer can, in principle, learn any abstract concept.

\begin{remark}[Applicability of Universal Approximation Theorems]
\label{rem:universal}
Fully Connected Networks of sufficient depth/width are universal approximants \cite{cybenko1989}. This may seem in contrast with the previous claim. This is due to the domain of the approximation and out-of-domain generalization: Let $p_\omega(x)$ be the encoding of a concept $\omega$. The universal approximation theorem guarantees that, given any \textit{compact} subset $C \subset I$, we can always find a network $f_w(x)$ such that $\|f_w(x) - p_\omega(x)\|_\infty < \epsilon$ on $C$. With further hypotheses %
it can also be shown that, given enough observations in $C$, we can learn such an approximation. However, existing universal approximation theorems give no valid guarantee anywhere outside of $C$. Indeed, as \Cref{prop:even-numbers} shows, the approximation $f_w(x)$ may be doomed to be no better than chance level outside of a bounded domain. Binding physical objects to digital identities hinges critically on the ability of associating unbounded entities to finite ones, beyond the reach of existing universal approximation theorems.
\end{remark}

\subsection{Concepts understandable by Differentiable Programming}
\label{sec:dp-undestanding}

In \Cref{sec:understanding} we have seen that a Gold Learner can understand any concept in a concept class by guessing (enumeration), as long as the concepts are encodable, enumerable and identifiable from their efflux. However, this comes at a prohibitive computational cost. Differentiable Programming makes the guessing more efficient but may never get to the correct guess. We now discuss the failure cases that are peculiar to Differentiable Programming and how they relate to the structure of a concept, or \textit{learning task.}

Beyond the limitations due to the architecture described in \Cref{prop:encodable}, 
here we study limitations due to local training procedures characteristic of most current methods based on stochastic gradient descent (SGD) and its variants. In order to arrive at a qualitative understanding of the complex relation between the optimization algorithm and the structure of the concept, manifest in the \efflux, we necessarily have to simplify some details. The upshot is that the main issue with Differentiable Programming, as it relates to learnability of abstract concepts, is not local minima of the loss function, but local minima in the \textit{structure of the concept}, \textit{i.e.}, its \efflux. Specifically, the empirical cross-entropy loss can often be minimized to zero training error for overparametrized models, but this in no way implies that the concept has been understood (Remark~\ref{rem:overfitting}).

While the exact description of what a complex architecture will learn on real data is challenging, in the following we aim to derive some qualitative understanding of how stochastic optimization algorithms interact with the structure of the task. In particular, we argue that while local minima in the loss function are not a concern for overparametrized models, local minima in their structure function are.

\begin{figure}
    \centering
    \includegraphics[width=0.9\linewidth]{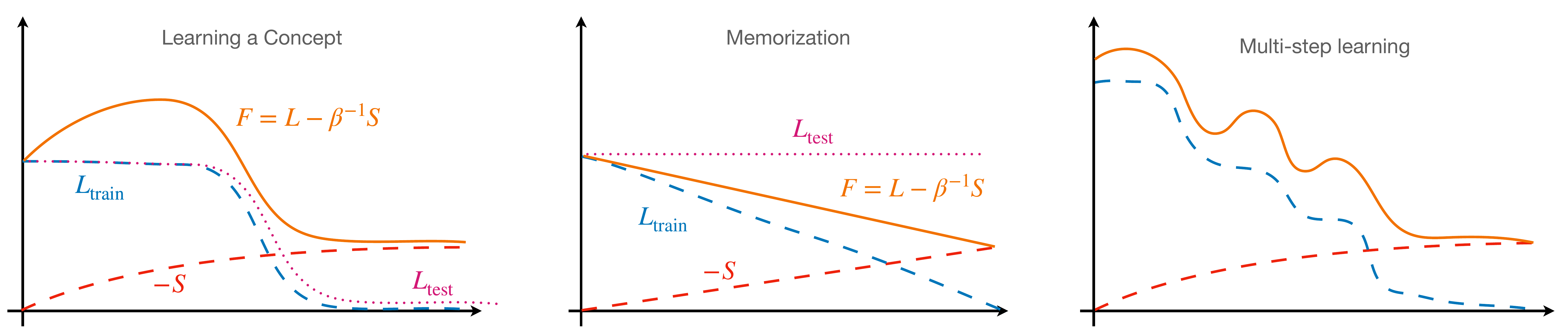}
    \caption{\textbf{Activation barrier, memorization, and multi-step local learning (cartoon).} \textbf{(Left)} Based on the theory described in this work, we show the expected qualitative behavior for the training loss $L_\D$ (blue), stored information (red), test error (dotted purple) and free energy (solid orange) and complexity $-S$ (red) during the learning of a concept. Initially, the model needs to store the data and information within cannot be immediately used to decrease the train error. The increase in complexity is required  to encode the concept. If the concept has been captured, both the training and test error  drop to zero. A local optimizer  minimizing free energy needs to be able to overcome the \textit{activation barrier} before seeing an effective benefit from learning the concept. \textbf{(Center)} Like learning a concept, memorizing the data results in a decrease of the training error (in blue), but the mechanism to do so is to use the increase in complexity to store the samples  ($-S$ red) at the same rate. This does not reduce the test error. If $T = \beta^{-1}$ is too low this may result in a global decrease in free energy, making the resulting memorization path preferable to the local optimizer, compared to having to overcome the activation barrier required for actual learning. \textbf{(Right)} Besides the initial activation, some concepts may have multiple stages where new barriers need to be overcome to progress to the next level. This results in a \textit{staircase structure} in the training error, which reduces the overall activation barrier at each step, making this class of concepts amenable to local learning, in that the activation effort for each step is within than the scope of the local learner.
    }
    \label{fig:free-energy}
\end{figure}

To understand the qualitative behavior of a (stochastic) learning algorithm, it is useful to analyze a training path across two dimensions: The training loss $L_\D(w)$ and the complexity of the learned solution, which we refer as (negative) \textit{entropy}\footnote{The sense in which a complexity measure for a fixed vector can be interpreted as entropy is developed in \cite{achille2021information}.} $-S(w)$. The latter depends on both the architecture used and the training algorithm, and increases during the training process, while the training loss decreases. While there is a wide variety of learning methods for large DNNs, most are variants of stochastic gradient descent (SGD), that share the same qualitative behavior and can, to first-order approximation, be interpreted as minimizing the Lagrangian \cite{achille2021information} (\textit{free energy})
\[
F(w) = L_\D(w) - T S(w),
\]
where $T$ is a \textit{temperature} parameter controlling the amount of noise introduced by the stochastic training algorithm (higher $T$ means more noise).

\added{\Cref{fig:free-energy} illustrates the fact that learning a concept with Differentiable Programming may involve overcoming a barrier in the free energy. The barrier may be decreased by reducing the temperature $T$, that is, reducing the noise in the optimization. However, this decreases the scope of the optimization, for instance the variance of the evolutionary scheme,  or the Fisher Information of SGD, making a local algorithm incapable of overcoming the barrier, leaving it with no option but to memorize the data to decrease the training error. this also makes memorization easier (\Cref{fig:free-energy}, center) as it does not penalize storing a large amount of information (negative entropy $-S$) that is only useful to classify an individual sample. Hence, reducing the temperature (e.g., reducing the learning rate or increasing the batch size too much) is not a viable option if a free-energy barrier is present.
The middle plot illustrates the easy way of minimizing the loss: Initially, the training error can be reduced  by fitting individual data (memorization), causing a linear increase in complexity, but with no benefit of generalization. This decrease in the loss is therefore premature and not desirable. Memorization can be prevented by increasing the temperature $T$, or equivalently increasing the noise level. This, however, creates a \textit{barrier in free energy,} which initially increases, rather than decreasing (left). If the optimization scheme can withstand the delay in gratification, eventually the benefit is a steeper decrease in the loss, since the  structure shared among different data points can be captured simultaneously. This is illustrated by the late onset but steeper drop in training error in the left cartoon. %

A stochastic algorithm can explore the landscape and eventually jump over a barrier. The average time needed to do so scales with $\tau = e^{\frac{1}{T} \Delta F}$, where $\Delta F$ is the height of the barrier. In particular, learning a concept with a taller  barrier will require exponentially more time, which in practice means that the algorithm may be unable to overcome the activation barrier.

The ideal case for learning a concept is when the \efflux naturally possesses  a ``step-wise structure,'' which enables a local learner to understand it \textit{incrementally} by overcoming multiple smaller barriers. In this case, the algorithm can learn intermediate solutions of lower complexity that still decrease the training loss (\Cref{fig:free-energy}, right). The activation barrier at each step is reduced, and in some cases small enough to be matched to the scope of a local algorithm, that can therefore overcome it step-by-step. Note that this process is unrelated to Curriculum Learning, that relates to scheduling a sequence of different learning tasks, and only refers to the structure of the \efflux originating from a single concept, and the associated learning task of understanding it from finite data.

\Cref{fig:free-energy} (right) qualitatively describes concepts that are learnable with local methods such as Differentiable Programming. We note that the plot depends on the training loss, which depends on the particular dataset $\D$ and the class of functions $f_w$, as well as on the particular algorithm and the nature and amount of stochasticity reflected in the free energy. While, intuitively, local learnability should be a property of the concept alone, regardless of the architecture, dataset and optimization, we have not found a language suitable for characterizing such ``incrementally-learnable concepts.'' This is an open problem for future investigation.  We point out, however, that tasks for which DNNs trained with SGD are successful, such as large-scale high-dimensional classification,\footnote{As an anecdotal example, consider classifying a dataset into airplanes vs. cartoon characters. Starting from random features, a simple color histogram can already provide a decrease in loss, where images with a lot of blue often have airplanes, but eventually plateau. Adding simple shape features helps discriminate cartoon images with blue colors but no long lines and sharp angles. But eventually, to correctly label airplanes painted with cartoon characters as airplanes, rather than cartoon characters, one would have to stand a rather sophisticated representation. This representation, however, can be built by the union of intermediate features each providing value through the reduction of error rate during training.} possess this structure, whereas relatively simple problems, such as hard subset selection with latent  regressor (as in many chicken-and-egg perception problems)  fall in the first class, and unsurprisingly remain a significant challenge for current methods. For example, we still do not have a general inductive amortization of RANSAC \cite{fischler1981random}.
}

\subsection{Awareness of a concept and validation}
\label{sec:awareness}

The fact that a concept can be encoded, learned, and understood does not imply that it can be communicated. Since encodings are non-unique and generally not accessible, there is no way for an external agent to know whether a particular learner has truly understood an abstract concept. However, if encodings are mapped to symbols in a dictionary shared among different learners, a learned concept can, in principle, be communicated externally. But while a symbol in a finite dictionary can be shared among different agents,  this does not imply that each agent associates that symbol to the same abstract concept.

\subsubsection{Continuous validation}
\label{sec:validation}

If the understanding of a concept cannot be communicated, how can we ensure that understanding has occurred? The concept of $\pi$ an be finitely encoded, but how do we know that a machine exposed to a finite sequence of its decimal has understood it? If we let the machine predict unseen decimals, we can verify that they are correct, but for any finite time, it is possible that -- at some point in the future -- the machine will start generating wrong decimals, thus revealing that it has not understood the concept after all. This is the teacher's conundrum, where student understanding cannot be determined directly, but only by administering a test. 
\begin{definition}[Falsifiability \cite{popper2005logic}] A concept $\omega$ with \efflux $I(\omega)$ is distinguishable from a concept $\omega'$ with \efflux $I(\omega')$ if at any point in time $t$, their truncations yield $\i^t(\omega) \neq \i^t(\omega')$. Until then, it is not possible to distinguish the two concepts from their \outpour{}s.
\end{definition}
Since the set of primary concepts that can be communicated through a shared dictionary is no larger than the dictionary, most concepts cannot be communicated or explained, and must instead be \textit{continuously validated}, and can only be considered understood until invalidated.

\paragraph{Persistently exciting dataset.}
Given a concept $\omega$ and a learner $G$, we call a dataset $D$ \textit{sufficiently (or persistently) exciting} if $G(D) = \omega$ recovers the correct concept. Note that it is not possible to verify whether a given dataset $D$ it is sufficiently exciting unless we already know the concept $\omega$. Because of this, an interactive learner has to engage in \textit{continuous validation}. The goal of an active learner is therefore equivalent to \textit{experimental design}, with falsification as the aim.

\paragraph{One-sided \efflux and anomaly detection}

An anomaly is an abstract concept related to the violation of a normal (null) hypothesis. Typically, in this scenario, a dataset consists only of normal data, and anomalies are detected as deviations from the normal model. Generally, the data is available only for the normal mode, whereby positive examples may lead to understanding a superset concept that is too general as the following example shows. 
A typical example is \textit{out-of-distribution detection} where samples are given that are assumed sampled from an unknown distribution, and each new sample must be tested against the same assumption.

\begin{example}[A concept that cannot be learned from one-sided examples.] 
Let $\omega_k = \{n \cdot k : n \in \N \} \subset \N$ denote the natural numbers that are multiple of $k$. Let the concept class be $\Omega = \{\omega_k : k = 1, 2, \ldots\}$, that is, the concepts are ``is multiple of $k$'' for some $k$. A positive efflux will then provide, in some order, a sequence of numbers that are multiple
of $k$. For example, $\i^t(k) = (k, 2 k, 3 k, \ldots, t k)$. We now show that a guessing learner $G$ could learn the wrong concept from such a one-sided \efflux: Fix any enumeration of the concepts of $\Omega$, and let $e_k$ denote the index of the concept $\omega_k$ in the enumeration. Consider the index $e_1$ of the general concept $\omega_1$. Since there are infinite concepts, there must be some $k_0$ such that $e_{k_0} > e_1$, that is, the specific concept $\omega_{k_0}$ comes after the general concept $\omega_1$. Let $G$ be guessing learner, that guesses the first concept $\omega_k$ compatible with all observed data, that is such that $\i_{t'} \in \omega_k$ for all $t' \leq t$. Then, for any $t$ we will always have $G(\i^t(k)) \neq \omega_{k_0}$ since $\omega_1$ is compatible with all data from $\omega_{k_0}$ (it is more general) and, since it comes before $\omega_{k_0}$ in the enumeration by construction, $G$ would always prefer selecting $\omega_1$ instead of $\omega_{k_0}$. Note that this result does not depend on the particular enumeration of the result used.
\end{example}

The above example relies on $G$ being a simple enumeration guesser. What if we create a more complex learner? In \cite{gold1967} it is shown that for \textit{any} learner we can generate an adversarial sequence of positive examples that confuses it into learning the wrong concept.\footnote{Under the  assumption that $\Omega$ contains \textit{all} finite concepts and at least an infinite one. Note that the infinite concept will then contain infinite more specific (finite) concepts reducing to a similar setting as the one we describe in the example.} But this leaves open the question of whether there is a learner that with high-probability will learn the right concept from one-sided data if the positive samples are random rather than adversarial. 

\subsubsection{Communication and explanation}

The fact that a concept can be encoded, learned, and understood does not imply that it can be communicated or ``explained.'' Communication is the process of linking different encodings of the same concept that reside on different devices or brains through a shared dictionary. A dictionary is a finite set of encodings of abstract concepts, or \textit{symbols}. 

\added{If a shared dictionary already exists, concepts represented by symbols in the shared dictionary can be communicated.  However, for a shared dictionary to emerge, agents must be immersed in a shared medium, and form independent representations of physical concepts that are simultaneously experienced through independent observations. In this case, there is always uncertainty on whether the different encodings of the \textit{finite shared observations} represent the same underlying physical concept, which we discuss in the next section.}

Encodings of abstract concepts are typically not accessible by an external agent, and even if accessible they are not meaningful to share. Building a shared dictionary,  whereby different encodings of abstract concepts are connected through shared observations, is the inverse process of induction, whereby different observations are connected through shared abstract concepts. \added{Since communication with humans would require not only a shared dictionary, but a shared language, we do not further explore the notion of ``explainability'' but discuss potential for further explorations in \Cref{sec:discussion}. It should be noted, however, that there are concepts that can be easily shared between humans and trained models, for instance the location of one pixel in an input image, which is why some image-based localization tasks are sometimes discussed under the guise of ``explainability.'' In the next section we expand on the role of the physical world in providing a ``shared language'' to agents immersed within.}

\section{Physical objects and concepts}

Physical objects are bounded regions of the three-dimensional ambient space. Examples include animals, plants, and artificial objects.
Physical objects are only experienced indirectly through measurement devices, or \textit{sensors}.  A sensor is a mechanism that maps physical objects onto data. \textit{Perception} 
is the process of forming abstract concepts from sensory data. We call the resulting abstract concepts \textit{physical concepts}.

\begin{definition}[Physical Concept]
A physical concept is an abstract concept having as \outpour all the data that sensors could produce given infinite time. \end{definition}

An example of physical concept is a bounded scene, experienced through a finite collection of images, which  can be encoded in the weights of a neural network, for instance a Multi-layer Perceptron, known as ``NeRF'' (Neural Radiance Field) \cite{mildenhall2020nerf}. Understanding the scene is then equivalent to learning a ``complete NeRF'' that can synthesize every possible image, indistinguishable from the one that a real camera would capture  from  different viewpoints under different illumination. Is learning a ``complete NeRF'' from finite collections of imges even possible?

The \outpour of a physical concept can be infinite even if the underlying object is not, because of variability in the measurements. We distinguish \textit{intrinsic variability},  due to characteristics or \textit{attributes} of the physical object (which are themselves abstract concepts), from \textit{extrinsic} variability, due to the sensor and the environment independent of the  object. The latter can be further divided into structured perturbations (\textit{nuisances}) due to known mechanisms or ``causes''  shared among all objects, and  unstructured perturbations (\textit{noise}) independent within and across sensor measurements. Noise cannot in general be controlled, \textit{i.e.}, purposefully set to a particular value. Nuisances can, under certain circumstances, be controlled through a deliberate action (or ``intervention'') specific to the sensor. To describe nuisances, we introduce the notion of a \textit{physical model}. A physical model is a mathematical description (hence itself an abstract concept) of a physical object specific to a sensory modality. A model is designed to describe phenomena known to affect physical measurements at the level of granularity relevant to the concept of interest. In the next example we describe physical models restricted to visual and acoustic sensors, further detailed in the appendix.

\begin{example}[Visual and acoustic models] In \Cref{sec:visual} we describe a basic visual model consisting of a sensor, represented as a function with planar domain quantized into discrete pixels and co-domain quantized into discrete levels. The level of a pixel is obtained by integrating the number of incident photons during a finite temporal interval. 
Physical objects are modeled as piecewise smooth multiply-connected surfaces embedded in Euclidean space $S \subset {\mathbb R}^3$,  supporting a reflectance function $\rho:S\rightarrow {\mathbb R}_+$ under static illumination, up to a contrast transformation $h_t$. Nuisances include the reference frame of the sensor $g_t \in SE(3)$, \textit{contrast transformations}, and \textit{occlusions}. Occlusions $\Omega_t \subset S$ are portions of $S$ for which the line-of-sight to the origin of the reference frame of the sensor $g_t$, or \textit{vantage point}, intersect $S$. Occlusions are a function of $S$ and $g_t$, $\Omega_t = \Omega(S, g_t)$. An image $\hat x_t$ can be written as a function of the scene $(S, \rho)$, assumed static, and nuisances $(g, h)$ as:
\begin{equation}
    \hat x_t = h_t\circ \rho \circ \pi \circ g_t \circ S; %
    \quad \quad S = S_1 \cup \dots \cup S_K
\end{equation}
that is related to a quantized datum $ x_t$ obtained from a visual sensor by $x_t = \hat x_t + n_t$, where the noise $n_t$ includes quantization error as well as all other unmodeled phenomena and $\pi$ is a canonical perspective projection.\footnote{The detailed notation is in the appendix; in particular, the reflectance function $\rho$ is assumed constant along the line connecting each point $p \in S$ to the origin, assuming a transparent medium, and therefore with an abuse of notation can be thought of as being defined on the image plane $\rho: {\mathbb R}^2 \rightarrow {\mathbb R}_+$. Another mild abuse of notation is the use of the symbol $\pi$ to denote the constant Pi, and the canonical central projection map, which can be easily distinguished from the context.}

In \Cref{sec:acoustic} we describe a rudimentary model of an acoustic scene with source $S_0:{\mathbb R}  \rightarrow {\mathbb R} $ with finite bandwidth $|{\cal F}(S_0)| < \infty$ where ${\cal F}$ denotes the Fourier Transform. Nuisances include source location $g_t\in {\mathbb R}^3$,  modulation of the range (spectral distortion) $h_t: {\mathbb R}\rightarrow {\mathbb R}$ and domain (time warping) $\rho_t: {\mathbb R}\rightarrow {\mathbb R}$ of the signal, and a finite number $K$ of additional artificial sources $S_k$ with $k = 1, \dots, K$
\begin{equation}
    \hat x_t = h_t\circ \rho_t \circ g_t \circ S;  %
    \quad \quad S = S_0 + \dots + S_K
\end{equation}
and the measurement $x_t = \hat x_t + n_t$ incorporates noise that aggregates aliasing effects from finite temporal sampling at rates below twice the bandwidth, inter-reflections (echo) from the ambient environment, in addition to all other unmodeled phenomena.

Acoustic models differ fundamentally from visual models in two ways: The first is the absence of the non-linear projection map $\pi$ that, combined with $g_t$, causes occlusions and scaling phenomena that make sampling space-varying. The second is the way in which independent sources combine,  by linear superposition instead of occlusion. These differences will have consequences in the binding of physical and digital identities, described in Claim~\ref{claim:physics-learnable}.  

\end{example}

\paragraph{Adversarial interventions: Spoofing,  mimicry, and cat-and-mouse games}

\added{A physical concept corresponds to a physical entity, for instance a particular object, scene or individual. Its observations $\i^t(\omega)$ up to any given time $t$ may not be sufficient to uniquely identify the underlying ``true (physical) concept'' (identity) $\omega$. In particular, there may be a different ``digital'' concept $\omega'$, with a different or no embodiment, that can produce an identical \efflux $\i^t(\omega)$ for all $t$. We call this phenomenon \textit{spoofing}. Synthetically generated data that are indistinguishable from physical measurements, such as DeepFakes, are a form of spoofing. This phenomenon is distinct from \textit{mimicry} which is an attempt to replicate the physical entity $(S', \rho') \simeq(S, \rho)$, through makeup, props, or other physical means of reproducing the physical likeness so as to generate the given efflux $\i^t(\omega)$ without altering the sensor or its data processing pipeline.

In general, given a finite dataset $\D = \i^t(\omega)$ for some $t<\infty$, it is not possible to know whether it is \textit{sufficiently exciting}, therefore equivalent to the infinite \efflux, and therefore sufficient to uniquely identify, or understand, the underlying concept $\omega$. However, an active learner who can control the efflux of the true (physical) concept $\i^{t+1}(\omega, u^t)$ through some intervention $u_t$ that influences the \efflux, can in theory generate a persistently exciting \efflux, gaining  an advantage on the spoofer who is then forced to re-learn, or update, the concept $\omega'$ to match the efflux $\i^{t+2}(\omega') =\i^{t+1}(\omega, u_t)$. This is not possible instantaneously, forcing the introduction of a discrete delay. Successful spoofing forces a new action $u_{t+2}$ in response, or $u_{t+1}$ in anticipation of a cat-and-mouse game. 

The {\bf key question} in this paper, as it pertains to the ability of associating physical entities to their digital identities, is whether such a cat-and-mouse game will go on forever, or terminate with a successful binding, or terminate with success for the spoofer. A corollary question is, if the process terminates with a successful binding, whether we can \textit{know} that has happened so there is no cat-and-mouse game, and we can cease continuous validation. 

{\bf Key conclusions:} In this section, using concepts and propositions introduced in previous sections, we reach the following conclusions: First, a passive observer cannot terminate a cat-and-mouse game, and will have to continuously increase the size of the dataset \textit{one step behind the spoofer}. Second, an active observer can terminate the cat-and-mouse game, in the sense that there is a finite time by which the physical concept (\textit{e.g.}, the identity of a physical object or agent) is bound to a unique abstract concept (\textit{e.g.}, a digital entity, or symbol) 
and therefore a spoofer is bound to fail forever after that point in time. However, in general we cannot know when that happens, so although the  cat-and-mouse game has ended, we still have to conduct \textit{continuous validation} to ensure that the binding hypothesis is not invalidated. 

As a corollary of results from previous sections, we can also conclude that a ``Complete NeRF'' cannot be inferred from finite data, so NeRFs are at most a very efficient  data structure for image interpolation,  but cannot ``capture the true scene,'' for instance its continuous geometry and photometry, nor generalize or extrapolate to novel scenes. This does not mean that NeRFs cannot generate compelling interpolations for human observers. This is true even if we assume that the scene is compact, and physical objects are enumerable, as we argue in \Cref{sec:physical}. 

To deduce the claims, we have to articulate some characteristics of active learners, which are not just their ability to control nuisance variability, but also to mitigate the effect of noise, or unstructured perturbations. For this, we use the simplified phenomenological models described in \Cref{sec:visual} and \ref{sec:acoustic}.}

\subsection{Physical concepts and enumerability}
\label{sec:physical}

In the following two remarks, based on established results by others, we argue that physical concepts can be considered enumerable, even if they pertain to continuous regions of space. While the continuum is an abstraction, the arguments that follow hold even in the limit where, for instance, images had infinite resolution and objects were continuous surfaces supporting infinite-dimensional reflectance functions.

\begin{remark}[Enumerating physical concepts]
In \Cref{sec:definitions} we assumed that the set of concepts $\Omega$ is enumerable, and in particular countable. This may seem to contradict the intuitive idea that a \textit{continuum} of physical concepts may exist. There are several ways to address this problem. First, one should keep in mind that the set $\Omega$ may be restricted to a particular set of physical concepts that are of interest to the learner. For example, it could be the finite collection of sets of objects for which we have names in a dictionary, or an infinite set of objects that are, however, generated by a mechanism with finite (but arbitrarily large) complexity, thus making $\Omega$ enumerable. A more theoretically stretched argument is based on the  Bekenstein bound \cite{bekenstein2020universal},  whereby any (finite) physical entity has bounded information complexity. In particular, an object $e$ of mass $m$ and radius $r$ has at most $H=2.57 \cdot 10^{43} \cdot m \cdot r$ bits of information. While that number seems so large as to be considered infinite for all practical purposes, it is nevertheless infinitesimal compared to the number of images that the same object could generate, which is in the order of $10^{150,000}$ even for sensors with modest (VGA) resolution, as we argue next. This, again, suggests that the set of all physical concepts $\Omega$ can indeed be considered effectively enumerable, unlike their \efflux.
\end{remark}
While acoustic, tactile and olfactory sensors are subject to attenuation and lower-bounded quantization, one could in theory take a picture of infinity by just pointing at the sky, or keep magnifying to see finer and finer details, with no known limit to the best of our current scientific knowledge. But even then, what matters for binding objects to symbols is not the data, but the maximal invariant function of the data to nuisance variability, which even for infinite-resolution images has been shown to be finite \cite{sundaramoorthi2009set}.

\begin{remark}[Finite \efflux]
In \Cref{sec:definitions} and \Cref{sec:iterative}, we operated under the assumption that the \outpour $I(\omega)$ of a concept $\omega$ may be infinite. Indeed, the asymptotic learnability of a concept would be trivial for concepts having finite \outpour{}: After a long enough time all data in the efflux will have been observed and identifying the concept becomes a matter of search. Given that the efflux corresponding to a digital sensor is always finite (e.g., the set $I$ of all possible 8-bit RGB images of size $256\times 256$ is finite) one may ask about the need to develop an asymptotic theory of learning. However, it should be noted that the finiteness of possible observations is largely theoretical: the set of all possible RGB $256\times 256$ images has cardinality $10^{150,000}$, much more than what could be observed in the life of the universe. This makes approximating the concepts as effectively ``infinite'' in the theoretical derivation a more faithful representation of the real use cases.
\end{remark}

\subsection{Limitations of acoustic sensors in identifying physical concepts} \label{claim:acoustic}

Consider a physical concept that represents an acoustic source $\omega = S_0$, that generates an \efflux $\i^t(\omega)$ measured by acoustic sensors. For an active sensor, we assume total control of an independent source $u = S_1$. 

\begin{proposition}
An active acoustic sensor that generates an \efflux $\i^t(\omega, u)$ is no more powerful than a passive sensor that generates $\i^t(\omega', 0)$ for some $\omega'$.
\end{proposition}
This follows directly from the model in Example \ref{sec:acoustic}, by observing that $\i^t(\omega, u) = \i^t (\omega + u, 0)$ for all $t$, and therefore $\omega' = \omega + u$.  This implies that there is no benefit in the use of active learners for acoustic concepts, unlike visual concepts as we described next.

\subsection{Active sensors can identify physical concepts}

The efflux is the collection of all possible measurements associated to a concept under all possible intrinsic and extrinsic variability. A finite efflux where some of the nuisances are known is called a \textit{registered dataset}, and an efflux where some of the nuisances are controlled is called a \textit{controlled dataset}. If all modeled nuisances can be controlled, the efflux is called \textit{completely controlled}. Unlike modeled nuisances, by definition noise cannot be controlled. Yet, an active sensor can control the \textit{statistics} of the noise. We call \textit{benign} a noise process that, after registering known variability, yields a residual that is white, zero-mean, homoscedastic with a unimodal distribution that is symmetric about the man. The residual of a model can typically be made benign by incorporating higher-order statistics of the noise into the model. The following claim follows from standard large-number statistics arguments.

\begin{claim}[An active sensor can mitigate noise]
Consider a physical model and an active sensor with a completely controlled efflux that yields a benign residual characterized by a scalar parameter $\sigma$ (\textit{e.g.}, noise variance). Then the effect of noise can be made negligible: For any given $\epsilon > 0$, there exists $T \le \infty$ such as temporal averaging of the registered dataset yields residual with average noise having variance $\sigma \le \epsilon$.
\end{claim}
How effective the mitigation, that is how close to zero $\epsilon$ can be chosen, depends on the control authority of the active sensor (including processing power) and the properties of its passive components. We illustrate the claim with a visual example drawn from astronomical imaging, and an acoustic example drawn from cochlear implants. 
\begin{example}[Mitigating optical noise] 
\label{example:noise}
In the presence of constant illumination, a visual sensor capable of mobility can control the vantage point to register an object so the residual  $ h\circ \rho \circ \pi \circ g_t \circ S - n_t$ is white and zero-mean. Averaging on a temporal window then reduces the variance of the noise $n_T = \frac{1}{T}\sum_{t=1}^T n_t$. This is customary in astronomical imaging, where $S, \rho$ is a celestial body while $g_t \in G$ is not just a rigid body motion but a more complex deformations due to atmospheric turbulence distortion \cite{lou2013video}.\footnote{In practice, selection sometimes works better than averaging (the ``lucky frame method'' \cite{law2006lucky}) due to the complexity and inhomogeneity of atmospheric turbulence distortion.} In this case, what is controlled is not the position of a pin-hole but the orientation of individual mirror elements, enabling Nobel-worthy imaging \cite{stolte2008proper}.
\end{example}
If the scene is not static, the example still applies so long as the control authority of the sensor includes the ability to change the temporal sampling rate, and the characteristics of the passive elements allow sampling at a sufficiently fast rate relative to the motion of objects within the scene.

\begin{example}[Mitigating acoustic noise] 
A coarsely quantized acoustic sensor and transducer, even with just two levels (a threshold), can sample the acoustic range arbitrarily finely if the threshold can be controlled and the signal sampled at sufficiently high rate in time. The threshold can be changed on a regular schedule or at random (a process sometimes referred to as  ``stochastic resonance'' \cite{gammaitoni1998stochastic}). This technique is used  with acoustic signal processing in cochlear implants.
\end{example}
The next examples emphasize the distinction between \textit{knowledge} and \textit{control} of nuisance variability since a passive sensor cannot reduce the effect of noise: A registered dataset can never be guaranteed to define a physical concept. But a completely controlled efflux can yield a finite dataset that defines a physical concept, through \textit{exploration} or experiment design.

\begin{example}[Mitigating occlusions]
Occlusions are the most salient nuisance of image formation. Objects that are not visible are  not directly manifest in the data. The sensor provides no information about occluded objects, and no passive observer can ``undo'' occlusions. However, an active observer capable of mobility can easily invert occlusions by moving the vantage point around the occluder, to reveal the occluded object. Mobility is strongly associated with cognitive abilities in the phylogenic trees, as the abstract of \cite{vaas2001binds} illustrates for Tunicates.
\end{example}

\begin{example}[Mitigating  spatial quantization in optical sensors]
\label{example:hyperacuity} Spatial quantization can be thought of as a form of occlusion, where details with spatial frequencies higher than the natural (Nyquist-Shannon) sampling frequency cannot be resolved. In addition to the quantization of the sensor array, there are also spatial resolution limits due to the characteristics of the lens, and to the physics of diffraction. However, even those can be mitigated to an extent, in part by controlling the sensor (\textit{e.g.}, by spatial jittering of the threshold, akin to level jittering in acoustic sensors with stochastic resonance, and control of the lens), and in part by exploiting the regularities of the visual world, in the phenomenon of \textit{hyperacuity} \cite{westheimer1981visual}. The human eye is sensitive to single photons, and can resolve structures beyond the Nyquist limit, by exploiting strong priors in visual discontinuities arising from regularities in the scene ({\em e.g.}, long continuous curves due to occluding contours, material boundaries, or cast shadows).
\end{example}

\added{
{\bf Physical concepts can be identified from their complete \efflux.} A passive sensor, in general, cannot generate a complete \efflux and, even if it did, it would take an unreasonable amount of time to wait until the dataset is sufficiently exciting. An active sensor, depending on the degree of control authority it can exercise over the data collection process, can significantly expedite the collection of a sufficiently exciting dataset. Note that ``active'' refers not just to the sensing element (say, the CCD array) but to the overall system, which may include portions of the environment (a cooperative user). Due to the limitations of acoustic and chemical sensors due to the linear superposition of the sources, the locality of tactile sensing, and the high-dimensionality and complex interaction with the physical environment afforded by optical sensors, visual sensing plays an important role in understanding physical concepts. The more control authority, the richer the dataset, the smaller the indistinguishable set of underlying concepts. For any (enumerable) collection of physical \textit{distinguishable} concepts, combining the results of previous sections, we can conclude the following:
}
\begin{corollary}[Physical objects are learnable with an active sensor] 
\label{claim:physics-learnable}
\label{claim:active}
Let $\omega_1, \dots, \omega_n ,\dots$ be a collection of distinguishable physical concepts, and $\D_t = \i^t(\omega_i, u^t)$ be a dataset of visual and other data, with $u_t$ the control variables that include vantage point and illumination, and optionally acoustic signals instructing an agent environment. There exists a set of instructions/challenges/controls $u^\tau$ for some finite $\tau$ such that $\i^\tau(\omega_i, u^\tau) \neq \i^\tau(\omega_j, u^\tau)$ for all $\omega_i \neq \omega_j$.
\end{corollary}

\added{Without an active sensor, with a finite codomain $Y$, the object can always be below the sensor resolution and therefore not even be manifest in the data. This can always be done for images by moving sufficiently far. No matter how high (but finite) the resolution, no matter how large (but finite) the physical object, nuisance variability can always make the object invisible by moving it sufficiently far away. An active sensor ensures that sufficiently exciting data is collected.} 

\subsection{Binding physical entities to digital identities}
\label{sec:signal-symbol}

Finally, drawing on all previous claims, we conclude that it is possible to bind physical entities to digital identities, although knowing so requires additional structure and a shared dictionary.

\begin{corollary}[Cat-and-mouse games]
Given an active sensor and a collection of physical concepts as in Corollary~\ref{claim:physics-learnable}, an adversary abstract concept $\omega'$ aiming to spoof one of the $\omega_i$ by adapting its efflux $\i^\tau(\omega') = \i^\tau(\omega_i, u^\tau)$ will eventually, for some finite $\tau$, fail in the sense of being unable to generate the same efflux for any time $t \ge \tau$. However, it is not possible to know that $\tau$ has past, and therefore the agent has to continue challenging the learner by providing continuous validation until falsified. 
\end{corollary}
We note that there is still an important qualitative difference between a cat-and-mouse game where the agent is always one step behind the spoofer, which is the destiny of a passive observer, and a cat-and-mouse game where the agent is always one step ahead of the spoofer, who has to keep changing its strategy in response to persistently exciting challenges, while remaining within the discrete time available between temporal samples. We discuss more of these implications next.

\section{Discussion}
\label{sec:discussion}

We have explored the possibility of ``truly capturing,'' or understanding, physical concepts from finite sensory data. While there is no truth in data, we have shown that certain classes of abstract concepts can be understood by interactively seeking data from the concept. But while a concept can be understood, a learner has no way to make the external world aware of such understanding unless a shared dictionary is available beforehand. Nonetheless, external agents can engage in a process of falsification by comparing data to predictions (sometimes called hallucinations) originating from the concept (\efflux). This has implications for the design of learning architectures: Stateless feed-forward models do not afford the ability to understand anything but trivial concepts. It also has implications for learning methods, for classical inductive inference separated from a learning phase, cannot guarantee having captured a concept. Finally, our analysis has implications on the age old signal-to-symbol barrier, for it shows that establishing an association between infinite data and finite symbols is possible, although it may not be done in full generality with current (local) learning methods.

Our analysis shines a spotlight on the important role that the structure of the data plays in understanding the underlying concept. Local learning appears to be a substantial limitation, but in reality it is mitigated by the high dimensionality of the models. When learning a concept, ``getting stuck in local minima'' is not the main concern, for even a local procedure can easily find the global minimum for the parameters of the model by simply storing the data within. Instead, the challenge is to match the architecture and optimization to the structure of the data. This is still largely accomplished by human exploration.

Physical concepts are of particular interest in our analysis, since recent developments in generative models have put into question the ability to associate physical entities, such as individuals, with digital identities such as their online personas. We have shown that such association is possible, but requires persistent vigilance to test falsification. On the other hand, the excitement around NeRFs has led to suggestions that they can ``generalize,'' whereas our analysis concludes that, due to their being implemented with feed-forward architectures trained with passively sampled data, they are just a very effective data structure for image interpolation. They capture regularities of the images, but not of the structure of the ``true physical scene.''

\paragraph{Limitations of the analysis}

In order to render the analysis tractable we have had to abstract many of the details of current learning methods to what may appear cartoons. Technically, stochastic gradient descent (SGD), the method of choice for training deep neural networks (DNNs), does not satisfy the definition of local learning, since each individual step can in principle be arbitrarily large, albeit with vanishingly small probability. Furthermore, SGD is not even stochastic, since the residual from the true gradient is known at each step. Nonetheless, our definitions capture what we consider the critical feature of SGD, which is to enable -- through choice of hyperparameters -- to match the optimization to the ``step-wise'' structure of the data underlying certain concepts, where increased accuracy demands trigger the model to represent new features atop those already captured, creating a representational hierarchy. 

The same caveats apply to computational architectures, where we have abstracted convolutional architectures to feedforward functions, even though the same characteristic is shared by other architectures (for instance shallow networks or fully connected ones) that do not possess equally strong inductive biases. While their representational powers are undifferentiated, and captured by universal approximation theorems, the ability to reach a meaningful representation with local learning is strongly dependent on the architecture, which explains the interest in the topic and the Cambrian explosion of activity around architecture search and deliberate architecture design. 

\paragraph{Applicability}

Our analysis can be used to guide the design and interpretation of the results of training deep networks with local learning algorithms such as SGD, not to derive quantitative performance bounds for a particular architecture trained on a particular method for a particular task. Some of the claims we present may seem obvious to some readers, yet wrong to other readers. The purpose of this paper is split the discussion into two levels: Formally, our work allows deriving conclusions that can be validated analytically, rather than empirically as common practice. Our analysis of course cannot supplant empirical validation. Instead, it aims to complement and guide its interpretation. Beyond the formal aspect, our analysis is meant to force the explicit statement of assumptions underlying the claims, often implicit or hidden in the description of current methods, placing undue burden in the experimental validation which is always based on finite dataset with limited falsifying power. 

\paragraph{Explainability} Explainability is a term that conflates several issues relating to lack of understanding of the functioning of large parametric models. Some stem from the quest for \textit{causal} interpretations of the outcome of inference. Causality relates to counterfactual analysis of how the inference would have changed had some controllable action be taken, which was not. This says something about the model, not the concept nor the outcome of inference. Another aspect of explainability has to do with uncertainty: We want an explanation because we do not trust the inference. This would require defining uncertainty, which we have steered clear of in this manuscript, although see \cite{achille2021information}.  Here, we have restricted the use of the term ``explainability'' to communication: Different agents have different representations of a concept that they may have jointly experienced, and if these representations are mapped to a symbol in a shared dictionary, the symbol can serve for a proxy of the experience and used to connect the encodings. In this view, an explanation depends on the agents involved and requires a shared dictionary. In particular, an explanation for human consumption requires mapping to a human dictionary. This is more than joint training of a language model as common on multimodal Transformers and also include symbols that are geometric (\textit{e.g.}, particular locations or subsets of the image plane) yet meaningful to humans. Symbolic association is, however, not a guarantee of explainability, as the phenomenon of visual pareidolia \cite{liu2014seeing} or Rorschach's figures \cite{klopfer1942rorschach} illustrate. Finally, whereas symbolization and abstraction are unique under the assumption of indistinguishability, explanations are not -- as they depend on the shared dictionary, and the ability of the recipient to map it to the correct encoding, which cannot be determined without independent validation through the \efflux. 

\paragraph{Evolution of architectures and learning methods} Feed-forward architectures can only encode trivial concepts, and therefore can be used at most for data pattern recognition, not for abstraction or symbolization. Recurrent architectures are most powerful but also most difficult to train \cite{pascanu2013difficulty}. Transformers present a popular alternative where the input token are co-opted to implement a pseudo-recursive mechanism, but carrying it to full fruition would require infinitely long sequences of tokens.  

\added{These developments have implication for future research directions. First, they cast doubt of research programs hoping to develop a {\em homogeneous} ``universal architecture'' or a ``model to rule them all,'' especially if such architectures are feed-forward encoder-decoder networks. The same goes for Transformers, that cannot encode critical combinatorial components such as hard subset selection as required for representing occlusions. Recurrent architectures that maintain a finite memory, with explicit recursion mechanisms, are more promising but current learning methods are not effective at training them. More likely, different tasks will need specialized architectures or modules, thus shifting the design process from the features -- which Deep Learning and Differentiable Programming helped automate -- to the architecture, where ``graduate student descent'' remains the method of choice. Architecture search, as currently practiced, consists of local learning transposed to known architectural elements, thus again just shifting the problem one level up, where current procedures cannot truly discover beyond what they were engineered for. }

\added{Finally, we point back to the beginning, that is the data. In an inductive learning procedure, the data specifies the task, whose representation in the trained model contains, in general, a strict subset of the information within. In particular, current architectures can only encode knowledge representable in the language of first-order logic. Humans are capable of understanding concepts that require higher-order logic. It is intriguing that there is speculation, and some evidence, that human brains may encode second-order logic \cite{pantsar2021descriptive}. However, it should be clear that data-driven methods to discover laws, rules, causal factors and constraints, including laws of physics, can understand no more than what they can encode, which is at present limited. At the opposite end of the spectrum, claims of ``taskless learning'' where creativity is fostered rather than performance in a specific regression or classification problem, are moot in the context of this paper since creativity is a viable task in our nomenclature.
}

\added{It is clear that these are still the early days of Artificial Intelligence, and while  developments in the last decade already enable beneficial applications, along with some risks, the road ahead to true understanding is long.}

\bibliography{bibliography}
\bibliographystyle{plain}

\appendix

\section{Physical models}

Among all possible models of physical objects, we describe two simple ones based on visual and acoustic sensors, to emphasize the qualitative differences between the two sensors and their implication in the learnability of the underlying physical concepts. This is relevant to the establishment of a connection between physical entities and their digital identities.

\subsection{A basic visual model}
\label{sec:visual}

A visual sensor is a two-dimensional array of sensing elements (pixels) distributed on a finite lattice (image plane) each measuring the incident electromagnetic energy (irradiance) through a system of refractors (lenses), in a particular band of the electromagnetic spectrum (color), quantized into a finite number of levels (pixel values). At each instant of time $t$, the sensor produces an \textit{image}, which is the level measured at each pixel, obtained by truncating the irradiance integrated in a finite time interval (temporal sampling) in each color band. The set of possible images is finite, but even for small images larger than the number of particles in the universe. 

To describe the process of image formation, we model the ambient space as a collection of piecewise smooth multiply-connected surfaces embedded in Euclidean space $S \subset {\mathbb R}^3$,   supporting a reflectance function $\rho:S\rightarrow {\mathbb R}_+$ under static illumination and ignore non-Lambertian effects due to subsurface scattering, transparency, translucency, and inter-reflections. This is equivalent to considering surfaces as self-luminous and assuming their \textit{reflectance} equal to their \textit{radiance} up to a contrast transformation $h_t$; \textit{shape} $S$ and radiance $\rho$ can then be modeled as finitely-parametrized functions.  Nuisances include the reference frame of the sensor $g_t \in SE(3)$, \textit{contrast transformations}, and \textit{occlusions}. Contrast transformations $h_t:{\mathbb R}\rightarrow {\mathbb R}$ are monotonic continuous transformations of the codomain of $\rho$, which model global illumination changes. Occlusions $\Omega_t \subset S$ are portions of $S$ for which the line-of-sight to the origin of the reference frame of the sensor $g_t$, or \textit{vantage point}, intersects $S$. Occlusions are a function of $S$ and $g_t$, $\Omega_t = \Omega(S, g_t)$. An image of the scene $\hat x_t$ at a particular pixel location, under nuisances $h_t, g_t, \Omega_t$ is obtained by composing the shape $S$ with the vantage point $g_t$, with a central (perspective) projection $\pi: {\mathbb R}^3 \rightarrow {\mathbb P}^2$, restricted to the intersection of the projection rays from the pixel element through the origin of $g_t$ with the closest point on $S$. The radiance $\rho$, integrated on the surface element intersecting the solid angle that subtends the pixel (frustum), is composed with the contrast transformation $h_t$ to yield a level $\hat x_t$ 
\begin{equation}
    \hat x_t = h_t\circ \rho \circ \pi \circ g_t \circ S; %
    \quad \quad S = S_1 \cup \dots \cup S_K
\end{equation}
that is related to a quantized datum $\hat x_t$ obtained from a visual sensor by $x_t = \hat x_t + n_t$, where the noise $n_t$ includes quantization error as well as all other unmodeled phenomena. Occlusions are implicit in the application of the central projection $\pi$ of the shape $S$ through the viewpoint. The viewpoint can be controlled by a moving agent, in which case $g_t$ is its rigid \textit{motion}, although we also consider unordered collections of images taken from different vantage points, or views, so long as they belong to the \textit{same scene} in the sense that the frustra intersect. A simply connected components of $S$ and the restriction of $\rho$ to that component is \textit{model of an object}.

Note that we consider objects as entities that exist in the ambient world (physical objects), or as a mathematical description of a subset of Euclidean space (object models), or as abstract concepts encoded in the memory of a computer (physical concepts). We also consider concepts that are represented as elements of a dictionary in the human language, or \textit{semantic labels}, that exist in the brain of a human annotator looking at images. We do not, however, consider objects to exist in an image, where the only objects to be found are pixels.  

\subsection{Acoustic model}
\label{sec:acoustic}

An acoustic scene is a  source $S_0:{\mathbb R}  \rightarrow {\mathbb R} $ with finite bandwidth $|{\cal F}(S_0)| < \infty$ where ${\cal F}$ denotes the Fourier Transform. Nuisances include source location $g_t\in {\mathbb R}^3$,  modulation of the range (spectral distortion) $h_t: {\mathbb R}\rightarrow {\mathbb R}$ and domain (time warping) $\rho_t: {\mathbb R}\rightarrow {\mathbb R}$ of the signal, and a finite number $K$ of additional artificial sources $S_k$ with $k = 1, \dots, K$
\begin{equation}
    \hat x_t = h_t\circ \rho_t \circ g_t \circ S;  %
    \quad \quad S = S_0 + \dots + S_K
\end{equation}
and the measurement $x_t = \hat x_t + n_t$ incorporates noise that aggregates aliasing effects from finite temporal sampling at rates below twice the bandwidth, inter-reflections (echo) from the ambient environment, in addition to all other unmodeled phenomena.

Acoustic models differ fundamentally from visual models in two ways: The first is the absence of the non-linear projection map $\pi$ that, combined with $g_t$, causes occlusions and scaling phenomena that yield space-varying quantization. The second is the way in which independent sources combine,  by linear superposition instead of by selection (visibility). Nonlinear higher-order phenomena are manifest in both modalities due to the complex interaction of the sensor with ambient space, but due to their multiple independently compounded causes, they are lumped into unstructured noise in the models above. More sophisticated models can be designed to incorporate other known phenomena. We call concepts arising from acoustic sensors \textit{sounds}.

\section{Extended discussion}

\begin{description}

\item[If neural networks are universal approximating functions \cite{cybenko1989}] {\bf how can they \textit{not} learn the solution function to a given optimization problem?} The fact that a computational architecture can represent a particular function does not mean that there is a local algorithm to learn its parameters. Remark~\ref{rem:universal} discusses the applicability of universal approximation theorems in the context of understanding abstract concepts.

\item[Doesn't local learning get stuck in local minima?] The limitations of local learning are not due to convergence to local extrema, as commonly believed. In an 
overparametrized model, a local learner can typically train to zero error. However, it can do so by storing the samples, which yields overfitting and poor generalization. The problem, then, is not convergence to local minima, but absence from structure in the \efflux of a concept. Such structure is what allows a local learner to find a path to what will, eventually, be a good solution, even though it can only move in small steps. Local minima are, therefore, blind alleys of the task, not of the loss function of the learner, that can only follow the structure in the data engendered by the concept. 

\item[How do visual models  differ from acoustic models?] Acoustic sources compose linearly due to the superposition of energy transport. Visual sources occlude each other. This causes not just non-linearity but discontinuities in data information, and irreversible information loss. Also, acoustic signals do not change spectral characteristics depending on the sensor, to first approximation, and there are upper and lower bounds to quantization that can be used to discretize the data, unlike vision where scaling due to changes of viewpoint cause changing quantization and absence of a lower bound on spatial frequencies. As a result, acoustic sources can be more easily spoofed, as discussed in \Cref{sec:visual}.
\item[How do visual models differ from language models?] Language is tokenized at the outset and has finite generative complexity, in the sense of arising from a finite collection of symbols, combined with a finite set of rules, subject to a finite set of possible perturbations due to nuisance variability. While the data can be infinitely long, infinite complexity only arises from infinite length. Visual data, on the other hand, captures a non-compact portion of physical space, extending to infinity, in a single snapshot, where changes of vantage point combined with spatial quantization of the sensor result in the data not having a natural lower bound for quantization.
\item[How do masked autoencoders work for different modalities?] Masked autoencoders are a special case of models for predictive learning under simulated nuisance variability. Instead of specifying a target concept, the learning task is specified by nuisance variability shared among tasks. For instance, in visual recognition viewpoint, illumination, and occlusions are nuisances. If one views a datum as a point in some space, and the same datum transformed via changes of viewpoint, illumination and visibility as a trajectory in the same space, then contrastive learning aims to collapse these trajectories to the smallest possible volume, while keeping trajectories associated to individual data far from each other. In the case of contrastive learning, such trajectories are simulated rather than sampled in the data. When the nuisances act as a group in the space of the data, there is no need to perform contrastive learning for the trajectories are obits, and their maximal invariant (the quotient) can be computed analytically rather than approximate from sample data \cite{sundaramoorthi2009set}. Masked autoencoders restrict the simulation to occlusions. In vision, masked autoencoders obscure regions of an image, in language masked autoencoders obscure words. The training loss is just prediction. Self-supervised learning consists of concocting simulated nuisances that approximate changes of illumination (colorization), viewpoint or permutations (puzzles, reorganization), or occlusions (masked autoencoders), typically in isolation rather than jointly. 

\item[Can Transformers understand all human language?] Likely, given the finite Kolmogorov complexity of human language. 

\item[Can Transformers understand the visual world?] Likely not, since one of the most basic tasks, which is localization in the presence of occlusions, requires solving robust inference under a preponderance of outliers, which is not learnable/amortizable by a transformer.
Nonetheless, a learner could backpropagate through a combinatorial optimization stage to learn features that reduce outliers to within the relaxable regime, thus making inference tractable eve in the presence of exponential complexity (although dependent on the domain and not easily generalizable).

\item[How large do Transformers need to be before they capture all of human language?] The information complexity of human language is apparently very small \cite{mollica2019humans}. Nonetheless, extreme overparametrization may be needed to overcome the critical learning phase described in \cite{achille2018critical}. This may explain why we are currently using models with hundreds of billions of parameters to capture phenomena that require few millions of bytes of information.

\item[Counting arguments are futile when the numbers at play are gigantic:] {\bf Claims using them are vacuous!} Indeed, we agree. Which is why we differentiate ``learning by enumeration'' in the style of Gold in \Cref{sec:representable} (which we call \textit{guessing}, rather than learning) from learning using current methods later in Section~\ref{sec:iterative}. This is also why we distinguish between what a network can \textit{represent} (\Cref{sec:representable}) from what it can actually learn (\Cref{sec:iterative}, and point to the wide gap between the two: While in theory a network that incorporates recursion can represent any abstract concept, in practice what can be learned with differentiable programming is a restricted sets of tasks that lend themselves to local optimization, and point to very simple optimization problems (\textit{e.g.}, robust linear regression) that are not effectively learnable with local methods. To further emphasize, in the discussion of physical concepts we also point out that, while Physicists have claimed that the physical world has finite information, the number of bits needed to encode a finite mass are exorbitant. More to the point, even if we could line up every particle in the universe as well as count the number of possible digital images (say 1Megapixel 16-bit color images, relatively small by today's standards), we would have multiple images to encode each particle, which is clearly nonsensical. On the other hand, resorting to probabilistic assumptions such as the existence of a distribution that generates the data, and the article of faith that future data will be drawn from the same, which subtends most of modern Machine Learning, is similarly groundless, as such an assumption cannot be falsified. Instead \Cref{sec:validation} points to what is doable, as well as how it is possible to know, which does require persistent monitoring, in line with current practice. 

\item[The argument that the physical world is finite] {\bf does not hold water, and even the finite quantization of digital images is pointless. If we consider the physical world as well as images as points in the continuum (infinite sensor resolution), would any of the previous considerations still hold?} \added{For infinite-resolution images undergoing infinite-dimensional deformations of their domain as a result of (local) viewpoint changes, and infinite-dimensional deformations of their range as a result of contrast changes, the underlying concept is the \textit{maximal invariant}. That is the concept that subtends the orbits under the product group of domain diffeomorphisms and range homeomorphisms. The orbits are the \efflux. It has been shown in \cite{sundaramoorthi2009set} that the maximal invariant of infinite-dimensional images (modeled as Morse functions) modulo domain diffeomorphisms and range homeomorphisms exists and is a discrete object, specifically an Attributed Reeb Tree constructed from the Morse-Smale complex, where the node attribute is its ordering. In this case, there is nothing to learn, although it is common practice to simulate the orbits through data augmentation and approximate the maximal invariant via so-called \textit{Contrastive Learning}. For viewpoint and contrast changes, in theory it is not necessary to see any data, for the maximal invariant can be computed rather than learned. What is left to learn is occlusions, which do not induce group transformations in data space and therefore make the orbit space trivial (collapses to a point). Contrastive learning for occlusion nuisances leads to \textit{masked autoencoders}, a common practice in many domains from languages to vision. For vision, since the number of possible occlusions is the power set of the set of pixels, and the resulting optimization is not amortizable, this is a difficult problem that cannot be learned away nor quotiented out at the outset. 

}

\end{description}

\end{document}